\documentclass[]{bytedance}
\usepackage[toc,page,header]{appendix}

\usepackage{minitoc}
\usepackage{amsfonts}
\usepackage{amssymb}
\usepackage{tabularx}
\usepackage{listings}
\usepackage{xcolor}
\usepackage{cancel}

\newtheorem{proposition}{Proposition}

\newtheorem{proof}{Proof}

\usepackage{tabulary,multirow,xspace}
\usepackage{fixmath,mathtools,nicefrac,mmstyle}
\usepackage{subcaption}
\usepackage{caption}
\usepackage{wrapfig} 
\usepackage[misc]{ifsym} 
\usepackage{colortbl}
\usepackage{multicol}
\usepackage[most]{tcolorbox}
\usepackage{pifont}
\usepackage{booktabs}
\usepackage{graphicx}
\usepackage{hyperref}
\usepackage{url}
\usepackage{multirow}
\usepackage[normalem]{ulem}
\useunder{\uline}{\ul}{}

\definecolor{codegreen}{rgb}{0,0.6,0}
\definecolor{codegray}{rgb}{0.5,0.5,0.5}
\definecolor{codepurple}{rgb}{0.58,0,0.82}
\definecolor{backcolour}{rgb}{0.95,0.95,0.92}
\definecolor{boxblue}{RGB}{57,89,163}
\definecolor{boxbluebg}{RGB}{230,237,250} 

\lstdefinestyle{mystyle}{
    backgroundcolor=\color{backcolour},   
    commentstyle=\color{codegreen},
    keywordstyle=\color{magenta},
    numberstyle=\tiny\color{codegray},
    stringstyle=\color{codepurple},
    basicstyle=\ttfamily\footnotesize,
    breakatwhitespace=false,         
    breaklines=true,                 
    captionpos=b,                    
    keepspaces=true,                 
    numbers=none,                    
    numbersep=5pt,                  
    showspaces=false,                
    showstringspaces=false,
    showtabs=false,                  
    tabsize=2
}
\definecolor{mygray1}{gray}{.95}
\definecolor{mygray2}{gray}{.9}
\definecolor{mygray3}{gray}{.95}
\usepackage{pifont}

\newlength\savewidth
\newcolumntype{x}[1]{>{\centering\arraybackslash}p{#1pt}}

\newcommand{\app}{\raise.17ex\hbox{$\scriptstyle\sim$}}

\makeatletter
\DeclareRobustCommand\onedot{\futurelet\@let@token\@onedot}
\def\@onedot{\ifx\@let@token.\else.\null\fi\xspace}

\makeatother

\makeatletter

\newcommand{\Rmnum}[1]{\expandafter\@slowromancap\romannumeral #1@}
\makeatother

\usepackage{xcolor}
\usepackage{graphicx}
\usepackage{amssymb}
\usepackage{pifont}
\usepackage{floatrow}
\usepackage{amsmath} 
\usepackage{float}
\usepackage{wrapfig}
\usepackage{multirow}
\usepackage{tcolorbox}
\tcbuselibrary{breakable, skins, raster}
\usepackage{listings}
\usepackage{listings}

\definecolor{commentgreen}{rgb}{0.1, 0.4, 0.1}
\definecolor{keywordblue}{rgb}{0.1, 0.1, 0.7}
\definecolor{stringred}{rgb}{0.7, 0.1, 0.1}

\lstdefinestyle{mystyle}{
    commentstyle=\color{commentgreen},
    keywordstyle=\color{keywordblue},   
    stringstyle=\color{stringred},
    basicstyle=\ttfamily\scriptsize, 
    breaklines=true,
    keepspaces=true,
    showstringspaces=false,
    frame=none,                     
    language=Python, 
}

\usepackage{xspace}
\usepackage{makecell}
\usepackage{booktabs}
\usepackage{amssymb} 
\usepackage{pifont} 
\usepackage{dblfloatfix}
\usepackage{enumitem}
\usepackage{multirow}
\usepackage{rotating}
\usepackage{array}
\usepackage[autostyle=true]{csquotes}

\usepackage{colortbl}
\usepackage{xcolor}
\definecolor{Red}{RGB}{192, 0, 0}
\definecolor{Blue}{RGB}{12, 114, 186}
\definecolor{Yellow}{RGB}{218, 169, 20}
\definecolor{lightyellow}{RGB}{255,255,153}

\definecolor{HighlightBlue}{RGB}{0, 100, 148}
\definecolor{HighlightRed}{RGB}{230, 57, 70}

\definecolor{LightRed}{HTML}{ffe0e0}
\definecolor{LightBlue}{HTML}{def5ff}
\definecolor{LightYellow}{HTML}{FFF6DB}
\definecolor{LightGreen}{HTML}{eff9f0}

\usepackage{hyperref}
\usepackage{url}
\usepackage{natbib}
\usepackage{pifont}
\usepackage{color, xcolor}

\usepackage{wrapfig,lipsum,booktabs}

\usepackage{colortbl}
\definecolor{lightyellow}{RGB}{255,242,204}
\definecolor{lightorange}{RGB}{251,229,214}
\definecolor{lightgreen}{RGB}{226,240,217}
\definecolor{lightblue}{RGB}{222,235,247}
\definecolor{lightgray}{RGB}{209,201,206}
\definecolor{deepgray}{RGB}{178,164,173}
\definecolor{deepblue}{RGB}{112,168,218}

\usepackage{threeparttable}
\usepackage{xspace}
\usepackage{graphicx} 
\usepackage{float} 
\usepackage{booktabs}
\usepackage{multicol}
\usepackage{multirow} 
\usepackage[ruled,boxed,linesnumbered]{algorithm2e}
\usepackage{caption}
\usepackage{cleveref}
\crefname{figure}{Fig.}{Figs.}  
\Crefname{figure}{Fig.}{Figs.}  
\crefname{table}{Tab.}{Tabs.}  
\Crefname{table}{Tab.}{Tabs.}   
\crefname{section}{Sec.}{Secs.} 
\Crefname{section}{Sec.}{Secs.}   
\crefname{appendix}{Appendix}{Appendix}   
\Crefname{appendix}{Appendix}{Appendix}

\usepackage{setspace}

\usepackage{tikz}

\newcommand{\secondbest}[1]{\underline{#1}}
\newcommand{\safeincludegraphics}[2][]{%
  \IfFileExists{#2}{%
    \includegraphics[#1]{#2}%
  }{%
    \PackageWarning{instructmole}{Missing figure #2; rendering a placeholder}%
    \fbox{%
      \begin{minipage}[c][0.25\textheight][c]{0.95\linewidth}
        \centering Missing figure: \texttt{\detokenize{#2}}
      \end{minipage}%
    }%
  }%
}

\usepackage{booktabs,threeparttable,tabularx,array,ragged2e,xcolor,siunitx}
\newcolumntype{L}[1]{>{\raggedright\arraybackslash}p{#1}}

\title{InstructMoLE: Instruction-Guided Mixture of Low-rank Experts for Multi-Conditional Image Generation}

\author{
\centerline{
    Jinqi Xiao $^{2,{\ddagger}}$\quad 
    Qing Yan $^{1,{\dagger}}$ \quad  
    Liming Jiang $^{1}$ \quad 
    Zichuan Liu $^{1}$ \quad
    \vspace{5pt}
} 
\centerline{
    Hao Kang $^{1}$ \quad
    Shen Sang $^{1}$ \quad
    Tiancheng Zhi $^{1}$ \quad
    Jing Liu $^{1}$ \quad
    \vspace{5pt}
}
\centerline{
    Cheng Yang $^{2}$ \quad
    Xin Lu $^{1}$ \quad
    Bo Yuan $^{2}$ \quad
    \vspace{-5pt}
}
}

\affiliation[1]{ByteDance Inc.}
\affiliation[2]{Rutgers University}
\contribution[\dagger]{Project lead}
\contribution[\ddagger]{Corresponding Authors}

\abstract{
    Parameter-Efficient Fine-Tuning of Diffusion Transformers (DiTs) for diverse, multi-conditional tasks often suffers from task interference when using monolithic adapters like LoRA. The Mixture of Low-rank Experts (MoLE) architecture offers a modular solution, but its potential is usually limited by routing policies that operate at a token level. Such local routing can conflict with the global nature of user instructions, leading to artifacts like spatial fragmentation and semantic drift in complex image generation tasks.
    To address these limitations, we introduce InstructMoLE, a novel framework that employs an Instruction-Guided Mixture of Low-Rank Experts. Instead of per-token routing, InstructMoLE utilizes a global routing signal, Instruction-Guided Routing (IGR), derived from the user's comprehensive instruction. This ensures that a single, coherently chosen expert council is applied uniformly across all input tokens, preserving the global semantics and structural integrity of the generation process. To complement this, we introduce an output-space orthogonality loss, which promotes expert functional diversity and mitigates representational collapse. Extensive experiments demonstrate that InstructMoLE significantly outperforms existing LoRA adapters and MoLE variants across challenging multi-conditional generation benchmarks. Our work presents a robust and generalizable framework for instruction-driven fine-tuning of generative models, enabling superior compositional control and fidelity to user intent.    
}

\date{\today}

\checkdata[GitHub]{\url{https://github.com/yanq095/InstructMoLE}}

\begin{document}
\maketitle

\section{Introduction}
The advent of powerful, open-source Diffusion Transformers (DiTs)~\citep{peebles2023scalable, esser2024scaling, flux2024} has unlocked unprecedented capabilities in generative AI, fueling a demand for highly specialized and compositional functionalities, from multi-subject composition to personalized content creation~\citep{labs2025flux, wu2025qwenimagetechnicalreport, wu2025omnigen2, liu2025step1x-edit}. Parameter-Efficient Fine-Tuning (PEFT), particularly Low-Rank Adaptation (LoRA)~\citep{hu2022lora}, has become the de facto standard for such customization~\citep{xiao2024coap}. However, LoRA's monolithic update structure conflicts with the demands of multi-task fine-tuning, leading to catastrophic forgetting as different task objectives interfere~\citep{biderman2024loralearnsforgets,han2024parameterefficientfinetuninglargemodels}. 


The Mixture-of-Experts (MoE) architecture offers a structured solution to this interference problem. As theoretically grounded by~\citep{li2025theory}, MoE mitigates catastrophic forgetting by diversifying its experts to specialize in different tasks, which in turn helps to establish, or at least preserve, task-specific expert circuits and balance the loads across them. In the domain of language and multi-modal models, MoE has been extensively studied, leading to established, often task-aware, routing strategies~\citep{wumixture,liu2025beyond,li2024mixlora,gou2023mixture,chenoctavius,wu2025routing,sun2025a,fei2024scalingdiffusiontransformers16}. 

In contrast, its application to diffusion transformers for multi-conditional image generation (e.g. image editing, multi-subject driven generation) remains comparatively underexplored. Initial works like ICEdit~\citep{zhang2025context} use LoRA-MoE adapters with input-dependent routing, but they do not specialize routing for globally instruction-consistent expert assignment across the generated image. This reliance on locally varying or input-state-dependent routing reveals a potential misalignment with the global nature of image generation instructions. A user's instruction for compositional generation, for instance, to create a scene with ``\textit{A baby crawling on the grass, a white horse grazing nearby, and a football helmet}'' (in Figure~\ref{fig:qual}), establishes a complex set of global semantic relationships. In contrast, token-level routing delegates expert selection to each local image patch independently. This uncoordinated decision-making process can lead to critical failures, often resulting in subjects that are incoherently rendered, spatially fragmented, or have their specified relationships and attributes ignored entirely.

To rectify this misalignment, we propose InstructMoLE, a MoLE framework for multi-conditional image generation built upon a novel, globally consistent routing policy. The core of our framework is Instruction-Guided Routing (IGR), which conditions expert selection entirely on the global semantics of the user's textual instruction. This mechanism ensures that for any given layer of the model, a single, unified ``expert council'' is chosen based on the instruction and broadcast to all spatial locations within that layer. This enforces processing consistency at each stage of generation, while critically allowing the model to recruit different sets of specialized experts across different layers, tailoring the computation to the varying levels of abstraction. However, the effectiveness of this globally-applied council hinges on the functional diversity of its constituent experts. To ensure this diversity and prevent representational collapse, we complement IGR with a novel output-space orthogonality loss. This regularizer encourages the learned experts to occupy distinct functional roles, thereby maximizing the compositional power of the selected council. Our main contributions are:
\begin{itemize}
    \item We identify a critical challenge in applying MoLE to instruction-based editing and generation: the inherent mismatch between local, token-level routing policies and the global, semantic scope of user instructions. To address this, we propose InstructMoLE, a framework centered on Instruction-Guided Routing (IGR), which aligns expert selection with the holistic intent of the instruction.

    \item We introduce a novel output-space orthogonality regularizer to explicitly promote functional diversity among experts. This technique complements standard load-balancing losses by directly penalizing representational redundancy in the expert outputs, thereby mitigating expert collapse and improving the model's compositional control.

    \item We provide empirical validation on the Flux.1 Kontext backbone across multiple challenging benchmarks for multi-conditional image generation. Our results demonstrate that InstructMoLE achieves strong performance, particularly in tasks demanding high compositional fidelity and adherence to complex spatial relationships. These findings establish that Instruction-Guided Routing (IGR) is more effective than traditional token-level routing policies in our controlled setting.
\end{itemize}

\begin{figure}[t] 
  \centering
  \safeincludegraphics[width=\textwidth]{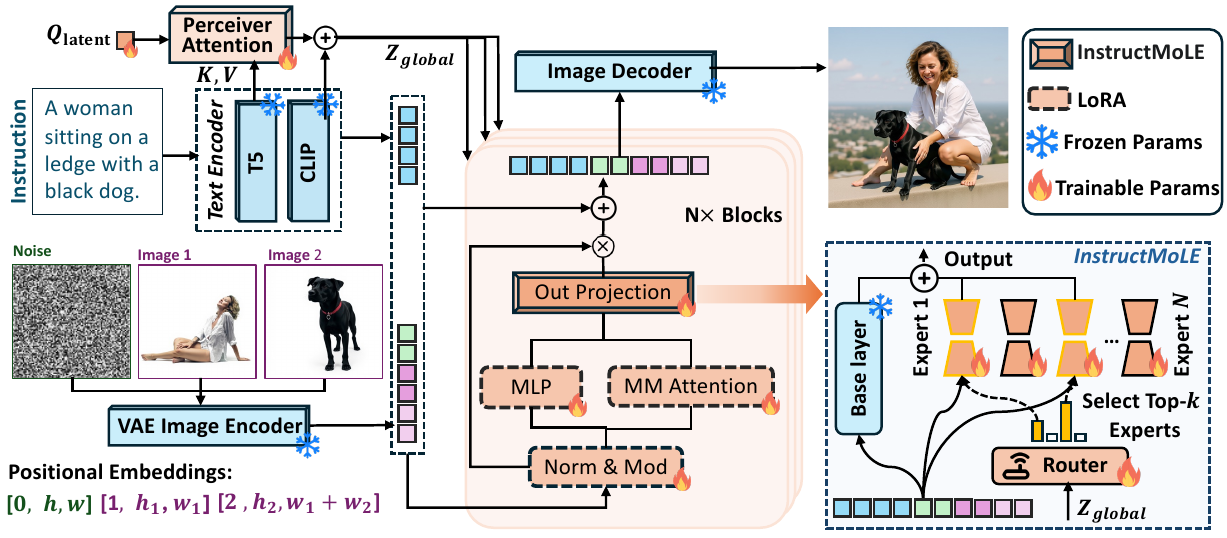}
  \caption{Illustration of the InstructMoLE framework. A global signal, ${\mathbf{Z}_\text{global}}$, is distilled from the user's instruction to guide a Router. The Router selects a single, consistent set of LoRA experts, which is then applied to all input tokens.}
  \label{fig:architecture}
\end{figure}

\section{Related Work}
\subsection{Conditional Generation with Diffusion Transformers}
The advent of Diffusion Transformers (DiTs) has marked a new era for generative modeling, demonstrating remarkable scalability and performance~\citep{peebles2023scalable, esser2024scaling}. A significant line of research has focused on enhancing their controllability for complex, instruction-driven tasks. Methods like DreamO~\citep{mou2025dreamo} and In-Context Edit~\citep{zhang2025context} have enabled sophisticated image customization and editing based on diverse user inputs. Others have introduced lightweight modules for flexible conditioning~\citep{zhang2025easycontrol, jiang2025infiniteyou} or tackled data bottlenecks in multi-subject scenarios~\citep{wu2025less}. Complementary work has explored compression techniques for attention-based vision models~\citep{xiao2023comcat}. While these works significantly advance the state of conditional generation, they do not explicitly align expert selection with a single global instruction representation. This leaves open the question of how to design modular adaptation that preserves global semantic consistency in multi-conditional generation.

\subsection{Mixture-of-Experts for Parameter-Efficient Fine-Tuning}
The Mixture-of-Experts (MoE) architecture, particularly in its low-rank form (MoLE), offers a powerful solution to the limitations of monolithic fine-tuning. By allocating specialized LoRA experts for different functions, MoLE can mitigate task interference and enhance model capacity efficiently. This paradigm has been extensively and successfully explored in language and multi-modal models for multi-task learning and instruction tuning~\citep{dou2024loramoe, gou2023mixture, chenoctavius}. Related efforts have also focused on compressing MoE models through inter-expert pruning and low-rank decomposition~\citep{yang2024moe}. Its application to vision, however, is more nascent. While initial works have validated the potential of MoLE for tasks like controllable visual effect generation~\citep{mao2025omni}, the full power of this architecture is unlocked only through a well-designed routing policy that determines which experts to activate for a given input.

\subsection{MoE Design: Routing Policies and Expert Diversity}
The design of an effective MoE system centers on two core challenges: the routing policy and the promotion of expert diversity. The predominant routing approach, inherited from language models, is token-level routing, where each token independently selects experts~\citep{shazeer2017outrageously, zhang2025context}. While advanced variants like Expert Choice routing improve dynamics~\citep{sunec}, they inherit the same fundamental limitation: local, per-token decision-making. This is misaligned with the holistic intent of user instructions in image generation. A second, orthogonal challenge is the common issue of ``expert collapse'', where experts become functionally redundant. Standard load-balancing only encourages uniform utilization, not functional diversity. Our work addresses these dual challenges head-on. We introduce Instruction-Guided Routing (IGR), a global policy that aligns expert selection with the user's instruction, and complement it with an output-space orthogonality loss that directly enforces functional diversity among experts.

\section{Methodology}

We present InstructMoLE, a Mixture of Low-rank Experts (MoLE) framework for multi-conditional image generation, illustrated in Figure~\ref{fig:architecture}. The framework champions a singular, top-down routing principle: a unified Instruction-Guided Routing (IGR) policy is applied consistently across all model layers. This is achieved through a Perceiver-style signal distillation for robust global guidance and an output-space orthogonality loss to maintain functional diversity among experts.

\subsection{InstructMoLE Architecture}
\label{sec:instructmole_arch}

Standard LoRA augments a frozen linear layer weight $\mathbf{W}_0 \in \mathbb{R}^{D_{\text{in}} \times D_{\text{out}}}$ with a single, static low-rank update~\citep{hu2022lora, xiao2023haloc}. We replace this static update with a dynamic, sparse mixture of $N$ expert LoRA modules $\{\mathbf{E}_i\}_{i=1}^N$. Each expert $i$ consists of a down-projection matrix $\mathbf{W}^i_{\text{A}} \in \mathbb{R}^{D_{\text{in}} \times r}$ and an up-projection matrix $\mathbf{W}^i_{\text{B}} \in \mathbb{R}^{r \times D_{\text{out}}}$, where $r \ll D_{\text{in}}, D_{\text{out}}$ is the rank. The function for expert $i$ is defined as $\mathbf{E}_i(\mathbf{X}) = \mathbf{X}\mathbf{W}^i_{\text{A}}\mathbf{W}^i_{\text{B}}$. The output of a MoLE layer is therefore:
\begin{equation}
    \text{MoLE}(\mathbf{X}) = \mathbf{X}\mathbf{W}_0 + \sum_{i=1}^{N} g_i(\cdot) \mathbf{E}_i(\mathbf{X}) 
    \label{eq:mole_general}
\end{equation}
where the gating weights $g_i(\cdot)$ are determined by a routing network; $\mathbf{X} \in \mathbb{R}^{B \times L \times D_{\text{in}}}$ is the input tensor with batch size $B$ and sequence length $L$. 

\paragraph{Token-level vs. Instance-level Routing.}
Conventional MoE models employ \emph{token-level} routing, where each of the $B \times L$ tokens independently selects experts from its local hidden state~\citep{fedus2022switch,shazeer2017outrageously,yuanexpert,sunec,zhang2025context}, resulting in a dense logits tensor of shape $\mathbb{R}^{B \times L \times N}$. The uncoordinated, per-token decisions create a tension with the inherently global nature of many instructions, leading to critical failures: (i) spatial fragmentation, yielding region-wise style and hue discontinuities; (ii) amplification of high-frequency noise through routing jitter; and (iii) weakened adherence to the user instruction.

To counteract these failure modes, we introduce an \emph{instance-level} policy that enforces a globally consistent expert assignment at each layer. Instead of fragmented per-token decisions, our router is guided by the global semantics of the user instruction to compute a \emph{unified expert council} for the entire instance. This council, represented by a compact logits tensor of shape $\mathbb{R}^{B \times 1 \times N}$, is then broadcast identically to all $L$ tokens within that layer. We term this mechanism Instruction-Guided Routing. By ensuring that a single, unified expert council is applied to all input tokens within a given layer, IGR directly prevents fragmentation and routing jitter, thereby preserving the semantic and structural integrity of the image.

\subsubsection{Instruction-Guided Routing (IGR) Policy}
\paragraph{Global Routing Signal ($\mathbf{Z}_{\text{global}}$).}
The IGR policy is conditioned on a global signal, $\mathbf{Z}_{\text{global}}$, which is distilled from the editing instruction, $\mathbf{I_c}$. To construct a signal that is both semantically robust and compositionally aware, we fuse two distinct text representations. Given an $\mathbf{I_c}$, a T5 encoder produces token-level features $\mathbf{H}_{\text{inst}} \in \mathbb{R}^{B \times L_{\text{inst}} \times D_{\text{inst}}}$, where $B$ is the batch size, $L_{\text{inst}}$ is the instruction length, and $D_{\text{inst}}$ is the T5 embedding dimension. Concurrently, a CLIP encoder provides a pooled embedding $\mathrm{CLIP}(\mathbf{I_c}) \in \mathbb{R}^{B \times D}$. While the pooled embedding offers a strong semantic anchor, it can overlook crucial compositional nuance. For example, in the instruction ``Change the woman's dress to red and the man's shirt to blue", a pooled embedding might average ``red" and ``blue", losing the critical association between the colors and their respective subjects. Conversely, token-level features, if naively averaged, risk diluting these key semantics.

We therefore employ a \textbf{Perceiver-style attentional bottleneck}~\citep{jaegle2021perceiver} to distill the rich, token-level information from $\mathbf{H}_{\text{inst}}$ into a compact summary before fusing it with the CLIP embedding. Concretely, we first project the T5 tokens into the $D$-dimensional CLIP space using a linear layer $\mathbf{W}_{\text{in}} \in \mathbb{R}^{D_{\text{inst}} \times D}$. A single, learnable latent query, $\mathbf{Q}_{\text{latent}} \in \mathbb{R}^{1 \times 1 \times D}$, is then used to iteratively query the projected instruction tokens via $S=2$ layers of Perceiver-style attention:
\begin{equation}
\label{eq:perceiver}
\begin{gathered}
\tilde{\mathbf{X}} \,=\, \mathbf{H}_{\text{inst}} \mathbf{W}_{\text{in}}, \quad
\mathbf{L}^{(0)} \,=\, \mathrm{tile}_B(\mathbf{Q}_{\text{latent}}) , \\
\mathbf{L}^{(s)} \,=\, \mathbf{L}^{(s-1)} + \mathrm{PerceiverAttn}(Q, K=\tilde{\mathbf{X}}, V=\tilde{\mathbf{X}})\,,
\end{gathered}
\end{equation}
where $\mathrm{tile}_B(\cdot)$ repeats the latent query $B$ times, $Q=\mathbf{L}^{(s-1)}$. The final latent, $\mathbf{L}^{(S)}$, serves as a distilled summary of the compositional details present in the T5 features. This summary is projected with $\mathbf{W}_{\text{out}} \in \mathbb{R}^{D \times D}$, normalized, and then additively fused with the CLIP embedding:
\begin{equation}
\label{eq:zglobal}
\mathbf{Z}_{\text{global}} \;=\; \underbrace{\text{LayerNorm}(\mathbf{L}^{(S)} \mathbf{W}_{\text{out}})}_{\text{distilled compositional summary}} \;+\; \underbrace{\mathrm{CLIP}(\mathbf{I_c})}_{\text{holistic semantics}}.
\end{equation}
Before fusion, we squeeze the singleton latent dimension so both terms have shape $\mathbb{R}^{B \times D}$. This design allows the distilled term to contribute fine-grained specificity from the instruction, while the CLIP embedding provides a robust, holistic semantic anchor, together forming a balanced and powerful global routing signal.

\paragraph{IGR Gating Mechanism.}
The IGR gating mechanism is instantiated independently at each MoLE layer. We formalize the computation by first describing the process for a single instance $b \in \{1, \dots, B\}$ within a batch, which corresponds to the input tensor slice $\mathbf{X}_b \in \mathbb{R}^{L \times D_{\text{in}}}$. The per-instance outputs, $\mathbf{Y}_b$, are subsequently stacked to form the full batch output $\mathbf{Y}$.

At a given $l$-th layer, the gating network $\mathcal{G}_{l}:\mathbb{R}^{D}\!\to\!\mathbb{R}^{N}$ generates a logit vector over the $N$ experts from a shared global instruction signal, $\mathbf{Z}_{\text{global}, b}$. The Top-$k$ experts are then selected:
\begin{equation}
\label{eq:igr_gate_compact_optimized_en}
\begin{aligned} 
(\mathcal{I}_{b}, \mathbf{w}_{b}) = \text{Top-}k(\mathrm{Softmax}(\mathcal{G}_{l}(\mathbf{Z}_{\text{global}, b}) ), k).
\end{aligned}
\end{equation}
Here, $\mathcal{I}_{b} \subset \{1, \dots, N\}$ is the set of indices for the Top-$k$ experts, and $\mathbf{w}_{b} \in \mathbb{R}^{k}$ contains the selected $\mathrm{Softmax}$ probabilities without re-normalization after Top-$k$ selection. $\mathcal{I}_{b}$ and $\mathbf{w}_{b}$ are then broadcast across all $L$ token positions in the sequence. The final output for the $l$-th layer, $\mathbf{Y}_b$, is computed by summing the output of a shared linear layer with the weighted sum of the selected expert outputs:
\begin{equation}
\label{eq:igr_forward_optimized_en}
\mathbf{Y}_{b}
=
\mathbf{X}_{b} \mathbf{W}_0
+
\sum_{j=1}^{k}
w_{b,j}
\cdot
\mathbf{E}_{\mathcal{I}_{b}[j]}\!\big(\mathbf{X}_{b}\big),
\end{equation}
where $w_{b,j}$ is the weight for the $j$-th selected expert, whose global index is $\mathcal{I}_{b}[j]$, and $\mathbf{E}_{\mathcal{I}_{b}[j]}$ denotes the corresponding low-rank expert.




\subsection{Theoretical Analysis}
\label{sec:theory}

We provide theoretical insights linking our design choices to established MoE principles.

\begin{proposition}[Expert Diversity]
\label{prop:diversity}
The orthogonality loss $\mathcal{L}_{\text{ortho}}$ encourages the expert outputs to form a linearly independent functional set within their span, satisfying the expert diversification requirement discussed above.
\end{proposition}

\begin{proof}
By minimizing $\mathcal{L}_{\text{ortho}}$, the expert output vectors $\{\mathbf{v}_i\}_{i=1}^N$ are encouraged toward mutual orthogonality. When these vectors are mutually orthogonal and nonzero, they form an orthogonal basis for their own $N$-dimensional span, ensuring each expert occupies a distinct functional role and reducing the risk of representational collapse.
\end{proof}

\begin{proposition}[Within-layer Spatial Consistency]
\label{prop:consistency}
At any fixed layer, IGR ensures structural spatial consistency: for any two tokens $x_i, x_j$ in the same instance, $R_l(x_i) = R_l(x_j)$, where $R_l(\cdot)$ denotes the routing decision at layer $l$.
\end{proposition}

\begin{proof}
IGR computes routing decisions solely from the global signal $\mathbf{Z}_{\text{global}}$, independent of local token states. The routing logits $\mathcal{G}_{l}(\mathbf{Z}_{\text{global}, b})$ are broadcast identically to all $L$ tokens within layer $l$, ensuring $R_l(x_i) = R_l(x_j)$ for all $i, j \in \{1, \dots, L\}$.
\end{proof}

These properties address token-level routing failures: Proposition~\ref{prop:consistency} removes within-layer spatial fragmentation, while Proposition~\ref{prop:diversity} motivates expert diversity. Expert councils may vary across layers, but each layer changes them synchronously over the full feature map.

\subsection{Training Objectives}
\label{sec:loss}

Our training objective is designed to achieve two goals simultaneously: high-fidelity image generation and the effective, diverse utilization of the expert pool. This is realized through a composite loss function, where the $\lambda$ terms are hyper-parameters:
\begin{equation}
    \mathcal{L}_{\text{total}} = \mathcal{L}_{\text{flow}} + \lambda_{\text{aux}}\mathcal{L}_{\text{aux}} + \lambda_{\text{ortho}}\mathcal{L}_{\text{ortho}}.
\end{equation}

\paragraph{Generation Accuracy ($\mathcal{L}_{\text{flow}}$).}
The primary objective, $\mathcal{L}_{\text{flow}}$, is the standard flow-matching or diffusion loss of the MM-DiT backbone. This ensures the model accurately synthesizes images that are well aligned with the user's instruction.

\paragraph{Expert Load Balancing ($\mathcal{L}_{\text{aux}}$).}
Following the standard practice for training MoEs~\citep{shazeer2017outrageously}, we employ an auxiliary load-balancing loss, $\mathcal{L}_{\text{aux}}$, to prevent the model from consistently favoring only a few experts. This loss is a function of two quantities calculated over a batch: $f_i$, the fraction of tokens routed to expert $i$, and $p_i$, the mean routing probability for that expert:
\begin{equation}
\label{eq:aux}
\mathcal{L}_{\text{aux}} \;=\; N \sum_{i=1}^{N} f_i \cdot p_i \,,
\end{equation}
where $N$ is the total number of experts.
\paragraph{Output-space Orthogonality Loss for Functional Diversity ($\mathcal{L}_{\text{ortho}}$).}
A critical failure mode in Mixture-of-Experts (MoE) models is \textit{expert collapse}~\citep{fedus2022switch}, where distinct experts converge to functionally redundant solutions, thereby nullifying the benefits of the mixture. The standard mitigation, an auxiliary load-balancing loss, only encourages that experts are utilized with similar frequency but provides no explicit mechanism to ensure their functional diversity.

To address this limitation directly, we introduce an \textbf{orthogonality loss} that penalizes the functional similarity between pairs of experts. Our approach operates on the raw, pre-gating outputs. For a given input batch $\mathbf{X}$, we first compute the output of every expert, yielding a set of tensors $\{\mathbf{Y}_1, \dots, \mathbf{Y}_N\}$, where each $\mathbf{Y}_i = \mathbf{E}_i(\mathbf{X}) \in \mathbb{R}^{B \times L \times D_{\text{out}}}$. To measure the functional similarity between any two experts $i$ and $j$, we flatten their respective output tensors into high-dimensional vectors, $\mathbf{v}_i = \text{vec}(\mathbf{Y}_i)$ and $\mathbf{v}_j = \text{vec}(\mathbf{Y}_j)$.

The orthogonality loss, $\mathcal{L}_{\text{ortho}}$, is then defined as the mean of the squared cosine similarities over all unique pairs of these expert output vectors:
\begin{equation}
    \mathcal{L}_{\text{ortho}} = \frac{1}{N(N-1)} \sum_{i \neq j} \left( \frac{\mathbf{v}_i \cdot \mathbf{v}_j}{(\|\mathbf{v}_i\|_2 + \epsilon)(\|\mathbf{v}_j\|_2 + \epsilon)} \right)^2.
    \label{eq:ortho_loss}
\end{equation}
Minimizing this objective encourages the vectors $\{\mathbf{v}_i\}_{i=1}^N$ to become mutually orthogonal, forcing the expert functions $\{\mathbf{E}_i\}_{i=1}^N$ to learn complementary roles and reducing expert collapse. The small $\epsilon$ stabilizes the loss under standard LoRA initialization, where some expert outputs can be near zero at the start of training; our implementation uses epsilon-stabilized L2 normalization. In practice, we compute the squared off-diagonal elements of the Gram matrix formed by the L2-normalized vectors $\{\mathbf{v}_i / (\|\mathbf{v}_i\|_2+\epsilon) \}_{i=1}^N$. Appendix~\ref{app:diversity} explains why these shared-input LoRA outputs are not independent random directions and reports empirical diversity measurements.

\section{Experiments}
\subsection{Experimental Setup}

\paragraph{Training Data.}
Our model's versatile editing capabilities are a direct result of its comprehensive training mixture. We curate a large-scale dataset by combining publicly available sources with a vast corpus of synthesized data, designed to expose the model to a diverse spectrum of conditional inputs. As illustrated in Figure~\ref{fig:train_data}, this includes reference-based tasks (e.g., face swapping, style transfer, re-lighting), multi-subject compositional generation, single-image editing, and spatially controlled generation from both dense (depth, Canny maps) and pose skeleton signals. This diverse training regimen enables a single, unified model to handle a wide array of editing modalities.

\paragraph{Evaluation Benchmarks.}
We evaluate all models on a suite of benchmarks, each targeting a distinct capability: OmniContext~\citep{omnicontext_benchmark_2025} for in-context generation, XVerseBench~\citep{chen2025xverse} for multiple subjects-driven generation, GEdit-EN-full~\citep{liu2025step1x} for single-image editing,  and MultiGen-20M~\citep{qin2023unicontrol} and COCO Pose 2017~\citep{lin2014microsoft} for spatially controlled generation. 

Tables~\ref{tab:omni}--\ref{tab:gedit} provide system-level context using released models on public benchmarks, while Table~\ref{tab:abla_rp} is the controlled comparison: all LoRA and MoE variants share the same Flux.1 Kontext backbone, data, schedule, and parameter budget.

In the following tables, the best results are in \textbf{bold} and the second best are \underline{underlined}. ($\downarrow$: Lower is better; $\uparrow$: Higher is better).

\subsection{Main Results}
\begin{figure}[t]
  \centering
  \safeincludegraphics[width=0.96\textwidth]{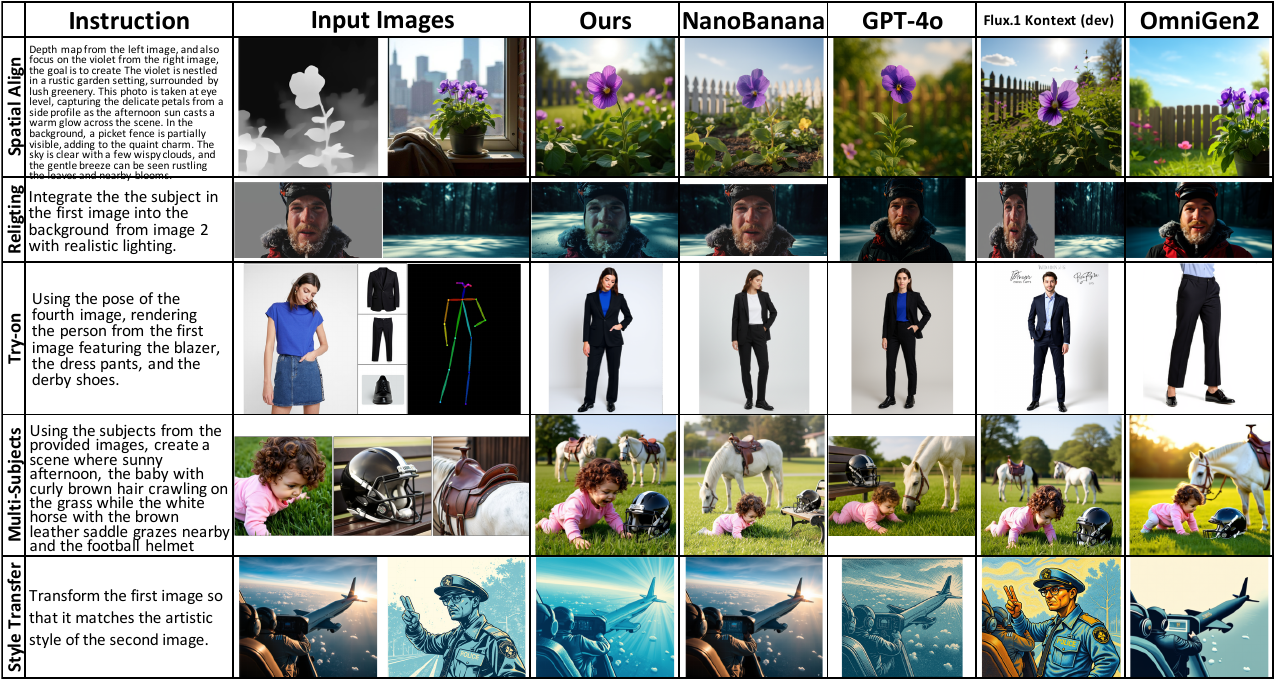}
  \caption{Qualitative comparison with state-of-the-art models.}
  \vspace{-0.8em}
  \label{fig:qual}
\end{figure}
\paragraph{Implementation Details.}
We fine-tune our model from the Flux.1 Kontext (dev) backbone for 100K steps on 64 NVIDIA H100 GPUs. We use $N=8$ experts, $k=4$, rank $r=32$, and apply MoLE to the feed-forward output projections: \texttt{ff.net.2} in double-stream blocks and \texttt{proj\_out} in single-stream blocks. Other trainable LoRA adapters use rank 256 on the image embedder, attention projections, modulation linears, and MLP projections. We set $\lambda_{\text{ortho}}=0.01$, load-balancing weights to 0.1/0.01 for instruction/token routing, and train with AdamW using learning rates $10^{-4}$ for LoRA/MoLE and $5{\times}10^{-5}$ for routers, with global batch size 256.

\paragraph{Efficiency Analysis.}
InstructMoLE achieves computational efficiency through global routing. On H100 GPUs, peak memory usage is 51.25 GB, lower than Expert Choice routing (51.88 GB), as global routing avoids token-wise gather-scatter operations. Inference incurs zero latency overhead compared with the baseline. The orthogonality loss does require all pre-gating expert outputs during training, but this dense expert computation is training-only and is not used at inference.

\paragraph{Qualitative Comparison.} 
Figure~\ref{fig:qual} qualitatively compares our model against state-of-the-art methods on several complex editing tasks, visually confirming its superiority. In tasks requiring precise spatial control like \textit{Spatial Align} and \textit{Try-on}, competing models often fail to respect geometric constraints, yielding misaligned objects or incorrect poses. In contrast, our model robustly adheres to both dense (depth) and sparse (pose) guidance. Our method also demonstrates a stronger grasp of compositional semantics and identity preservation, particularly in the \textit{Multi-Subjects} task where it correctly renders all subjects and their relationships. This advantage is most pronounced in the \textit{Swap Face} task, where InstructMoLE is the only model to produce a coherent result while others fail. This ability to disentangle and execute the spatial, semantic, and identity components of a complex instruction highlights the effectiveness of our approach.

\begin{table}[t]
\centering
\small
\caption{Quantitative comparison on OmniContext, evaluating prompt following and identity similarity (ID-Sim).}
{\renewcommand{\arraystretch}{0.85}
\resizebox{\textwidth}{!}{
\begin{tabular}{l|c|c|ccccc}
\toprule
\multirow{2}{*}{\textbf{Method}} & \multirow{2}{*}{\begin{tabular}[c]{@{}c@{}}\textbf{Prompt}\\ \textbf{Following}\end{tabular}} & \multicolumn{6}{c}{\textbf{ID-Sim}} \\
\cmidrule{3-8}
& & \textbf{Avg.} & \textbf{\begin{tabular}[c]{@{}c@{}}multi\\ character\end{tabular}} & \textbf{\begin{tabular}[c]{@{}c@{}}multi\\ character object\end{tabular}} & \textbf{\begin{tabular}[c]{@{}c@{}}scene\\ character\end{tabular}} & \textbf{\begin{tabular}[c]{@{}c@{}}scene\\ character object\end{tabular}} & \textbf{\begin{tabular}[c]{@{}c@{}}single\\ character\end{tabular}}    \\
\midrule
UNO & 5.63& 16.56& 20.65  & 15.16 & 16.10  & 8.49  & 22.41     \\
DreamO-v1.1 & 6.10& 15.47& 19.30  & 9.46  & 12.41  & 8.42  & 27.77     \\
OmniGen2& \textbf{7.58}  & 23.81  & 27.52  & 18.68 & 18.20  & 12.61 & 42.02    \\
Flux.1 Kontext (dev)& 6.24& \secondbest{36.29}& \textbf{50.08} & \secondbest{31.66}   & \secondbest{30.10}& \secondbest{14.97}   & \secondbest{54.65} \\
InstructMoLE& \secondbest{6.75} & \textbf{38.85} & \secondbest{45.12}& \textbf{34.03}& \textbf{34.59} & \textbf{20.04}& \textbf{60.47}  \\
\bottomrule
\end{tabular}}}
\vspace{-0.8em}
\label{tab:omni}
\end{table}

\paragraph{In-Context Generation.} As shown in Table~\ref{tab:omni}, InstructMoLE achieves the highest average ID-Sim score among the compared released models on the OmniContext benchmark. The benchmark evaluates instruction adherence (assessed by GPT-4.1) and identity preservation (ID-Sim)~\citep{Deng_2019_CVPR}. While OmniGen2 scores highest in \emph{Prompt Following}, it suffers from poor identity preservation, scoring 39\% lower than InstructMoLE (Ours) on the ID-Sim average. InstructMoLE demonstrates a superior balance, outperforming the strong Flux.1 Kontext baseline on both metrics. Although Flux.1 Kontext scores highest on the \emph{multi character} sub-metric, this is due to rendering artifacts where subjects are naively concatenated rather than composed into a coherent scene, as shown in Figure~\ref{fig:qual_omni} and Figure~\ref{fig:qual_omni_more}.

\paragraph{Multiple Subjects-driven Generation.}
We evaluate the model's ability to compose multiple subjects using the subject/reference images and prompts from XVerseBench~\citep{chen2025xverse}, a benchmark comprising 20 human identities, 74 unique objects, and 45 different animal species. We adapt the benchmark to our instruction-driven setting by providing the images and natural-language instruction directly as model inputs, rather than using the original XVerseBench condition-binding protocol. As detailed in Table~\ref{tab:xverse}, we assess performance using three key metrics: instruction adherence (DPG score~\citep{hu2024ella}), human identity preservation (Face ID similarity~\citep{Deng_2019_CVPR}), and object fidelity (DINOv2 similarity~\citep{oquab2024dinov}); the reported average is the arithmetic mean of these three metrics.

\begin{center}
\centering
\begin{minipage}[t]{0.46\textwidth}
\centering
\small
\captionof{table}{Quantitative comparison of multi-subject driven generation.}
\label{tab:xverse}
\resizebox{\linewidth}{!}{%
\begin{tabular}{lcccc}
\toprule
\textbf{Method} & \textbf{DPG} & \textbf{ID-Sim} & \textbf{IP-Sim}  & \textbf{Avg.} \\
\midrule
UNO                  & 74.35          & 31.96           & 53.07           & 53.13                            \\
DreamO-v1.1          & 89.92          & \secondbest{58.86} & 59.93        & \secondbest{69.57}              \\
OmniGen2             & \textbf{91.91} & 35.95           & \secondbest{60.93} & 62.93                            \\
Flux.1 Kontext (dev) & \secondbest{90.51} & 53.52       & 60.08           & 68.04                          \\
InstructMoLE         & 89.57          & \textbf{60.84}  & \textbf{62.81}  & \textbf{71.07}                   \\
\bottomrule
\end{tabular}%
}
\end{minipage}\hfill
\begin{minipage}[t]{0.49\textwidth}
\vspace{0pt}
InstructMoLE achieves the highest average score, outperforming DreamO-v1.1 by 1.50 points and Flux.1 Kontext by 3.03 points. Although OmniGen2 remains strongest on DPG, our method obtains the best ID-Sim and IP-Sim, showing that global instruction-aware routing improves subject fidelity without sacrificing instruction adherence.
\end{minipage}
\end{center}

\paragraph{Single-image Editing.}
As shown in Table~\ref{tab:gedit}, on the GEdit-EN-full benchmark, which evaluates performance across 11 distinct real-world editing categories by GPT-4.1, InstructMoLE again demonstrates its superior versatility by achieving the highest average score among the compared released models. Our model surpasses the strong Flux.1 Kontext baseline and exhibits a decisive advantage over other leading methods like OmniGen2, with an overall performance improvement of more than 13\% over the latter. This strength is particularly evident in its ability to handle both global and local manipulations with high fidelity. It obtains the best scores in fundamental tasks such as \emph{Color Alter}, \emph{Material Alter}, and \emph{Replace}, while remaining highly competitive across nearly all other categories.

\paragraph{Generalization to Unseen Tasks.}
On GenEval, InstructMoLE scores 64.14 versus 63.57 for the pre-trained Flux.1 Kontext baseline, suggesting that fine-tuning preserves the foundation model's original text-to-image capacity in this setting.

\paragraph{Robustness to Complex Instructions.}
On the Complex-Edit benchmark, which features highly complex, chain-of-thought instructions, InstructMoLE demonstrates robust generalization. It achieves superior perceptual quality (7.15 vs 7.08) and identity preservation (8.09 vs 8.01) compared to Flux.1 Kontext, confirming that the global routing signal maintains high fidelity even when processing complex and potentially ambiguous instructions.

\begin{table}[!htbp]
\centering
\caption{Quantitative comparison of single-image editing. The benchmark comprises 11 fine-grained editing categories reflecting practical user requests, with performance for each assessed by GPT-4.1. The corresponding qualitative results are shown in Figure~\ref{fig:qual_gedit}.}
\label{tab:gedit}
\resizebox{\textwidth}{!}{%
\begin{tabular}{l|ccccccccccc|c}
\toprule
\textbf{Model} & \textbf{BG Change} & \textbf{Color Alt.} & \textbf{Mat. Alt.} & \textbf{Motion} & \textbf{Portrait} & \textbf{Style} & \textbf{Add} & \textbf{Remove} & \textbf{Replace} & \textbf{Text} & \textbf{Tone} & \textbf{Avg} \\
\midrule
DreamO-v1.1  & 3.06   & 1.66& 2.35   & \secondbest{3.76}  & 3.24  & 3.34   & 2.16  & 0.55& 3.02 & 1.51  & 2.06  & 2.43  \\
ICEdit   & 2.73   & 6.00& 4.41   & 1.74& 2.14  & 5.19   & 4.41  & 1.53& 4.22 & 1.58  & 4.58  & 3.50 \\
OmniGen2 & \secondbest{6.99} & 5.10& 5.11   & \textbf{3.93}   & \textbf{4.59} & \textbf{6.88}  & 6.17  & 4.68& \secondbest{6.45}   & 4.04  & 6.05  & 5.45  \\
Flux.1 Kontext (dev) & \secondbest{6.99} & \secondbest{7.17}  & \secondbest{5.60} & 3.13& \secondbest{4.29}& 6.70   & \secondbest{6.90}& \textbf{6.92}   & 6.27 & \textbf{5.56} & \secondbest{7.14}& \secondbest{6.06}\\
InstructMoLE & \textbf{7.03}  & \textbf{7.46}   & \textbf{5.81}  & 3.29& 4.20  & \secondbest{6.79} & \textbf{7.03} & \secondbest{6.87}  & \textbf{6.53}& \secondbest{5.43}& \textbf{7.42} & \textbf{6.17} \\ 
\bottomrule
\end{tabular}
}
\end{table}
\begin{table}[t]
  \centering
  \caption{Ablation study of the IGR signal and the orthogonality loss.}
  
  \label{tab:abla_igr_orth}
  \resizebox{\textwidth}{!}{%
  \begin{tabular}{l|c|cc|ccc}
  \toprule
  \textbf{IGR Signal} & \textbf{Orthogonality Loss} & \textbf{Multi-Subject} (\(\uparrow\))  & \textbf{Single-Subject} (\(\uparrow\))  & \textbf{Depth RMSE} (\(\downarrow\)) & \textbf{Canny F1} (\(\uparrow\)) & \textbf{Pose F1} (\(\uparrow\)) \\
  \midrule
  $\mathrm{CLIP}(\mathbf{I_c})$ & w/o   & 64.54  & 6.12  & 35.67 & 38.12\% & 36.74\% \\
  $\mathrm{CLIP}(\mathbf{I_c})$ & w/   & 64.66  & 6.12  & 35.67 & 40.22\% & 36.79\% \\
  $\mathbf{Z}_{\text{global}}$  & w/o  & 65.65  & 5.97  & 33.74 & 38.92\% & 37.77\% \\
  $\mathbf{Z}_{\text{global}}$  & w/  & \textbf{65.68} & \textbf{6.15}   & \textbf{33.34}   & \textbf{41.51\%}  & \textbf{40.97\%}  \\ 
  \bottomrule
  \end{tabular}%
  }
  
  \end{table}

\subsection{Ablation Studies}
\paragraph{Implementation Details.} We train all ablations for 20K steps on 8 NVIDIA H100 GPUs. All MoLE routing variants in Table~\ref{tab:abla_rp} share the same experts, target modules, load-balancing objective, and orthogonality regularizer; only the routing policy changes.

\paragraph{Evaluation Metrics.} We conduct a comprehensive evaluation across three key domains. For \emph{multiple subjects-driven generation} and \emph{single-image editing}, we report average task scores on XVerseBench~\citep{chen2025xverse} and GEdit-EN-full~\citep{gedit_benchmark_2024}; in the ablation tables, \emph{Multi-Subject} and \emph{Single-Subject} denote these average scores rather than individual sub-metrics. For \emph{spatial alignment}, we evaluate on 500 samples for each modality, randomly sampled from the evaluation splits of MultiGen-20M~\citep{qin2023unicontrol} for depth and Canny edge control, and COCO Pose 2017~\citep{lin2014microsoft} for human pose control. The corresponding metrics are Root Mean Squared Error (RMSE) for depth, F1 score for Canny edges, and an F1 score based on Object Keypoint Similarity (OKS) for pose.

\paragraph{Routing Policy Analysis.}
We compare IGR against token-level policies in Table~\ref{tab:abla_rp}, using a high-rank LoRA as a strong baseline under a fair parameter budget. The standard Top-$k$ policy degrades performance, and Expert Race (ER)~\citep{sunec} is unstable on Pose F1. Expert Choice (EC)~\citep{yuanexpert} is stronger on Multi-Subject, suggesting a modest local-detail advantage, but IGR achieves much higher structural consistency. For newly acquired multi-conditional and spatial-control skills, coherent instance-level routing is therefore more important than adaptive token-wise expert selection.
\begin{table}[!htbp]
    \centering
    \caption{Ablation study on routing policies. ``Standard'' refers to the classic token-level Top-$k$ routing policy~\citep{fei2024scalingdiffusiontransformers16}. ``Standard+IGR'' applies the Standard policy to the early dual-stream blocks and IGR to the later single-stream blocks. ``IGR+Standard'' applies the reverse configuration.}
    
    \label{tab:abla_rp}
    \resizebox{\textwidth}{!}{%
    \begin{tabular}{c|c|cc|ccc}
    \toprule
    \textbf{Model} & \textbf{Routing Policy} & \textbf{Multi-Subject} (\(\uparrow\)) & \textbf{Single-Subject} (\(\uparrow\)) & \textbf{Depth RMSE} (\(\downarrow\)) & \textbf{Canny F1} (\(\uparrow\)) & \textbf{Pose F1} (\(\uparrow\)) \\
    \midrule
    Flux.1 Kontext (dev) & -   & 63.87  & 6.06& 63.44& 13.72\%   & 0.95\%   \\
    LoRA ($r=256$)   & -   & 64.68  & 6.06& 35.92& \secondbest{40.43\%} & 31.94\%  \\ \hline
    \multirow{6}{*}{\begin{tabular}[c]{@{}c@{}}MoLE\\ ($r=32,\,N=8,\,k=4$)\end{tabular}} & Standard& 64.56  & 6.07& 38.64& 28.06\%   & 19.50\%  \\
     & Expert Choice (EC)  & \textbf{65.81} & 6.07& 34.00& 38.55\%   & 31.47\%  \\
     & Expert Race (ER)& 55.83  & \secondbest{6.11}  & 102.67   & 16.73\%   & 0.00\%   \\
     & IGR+Standard& 65.17  & \textbf{6.15}   & \secondbest{33.43}  & 38.33\%   & \secondbest{33.77\%}\\
     & Standard+IGR& 64.86  & 6.07& 36.32& 33.33\%   & 9.57\%   \\
     & IGR (ours)  & \secondbest{65.68}& \textbf{6.15}   & \textbf{33.34}   & \textbf{41.51\%}  & \textbf{40.97\%} \\ 
    \bottomrule
    \end{tabular}%
    }
    \end{table}
\paragraph{Analysis of IGR Components.}
We conduct a fine-grained ablation on the core components of our IGR policy in Table~\ref{tab:abla_igr_orth}. The results provide three validations. First, adding the orthogonality loss to the CLIP-only baseline (Canny F1: 38.12\% $\rightarrow$ 40.22\%) demonstrates its contribution independent of the routing signal. Second, replacing the raw CLIP embedding with our fused signal, $\mathbf{Z}_{\text{global}}$, enhances performance across metrics (Multi-Subject: 64.54 $\rightarrow$ 65.65), confirming the necessity of distilling compositional details from token-level features. Third, combining both components yields the best performance (Canny F1: 41.51\%, Pose F1: 40.97\%), confirming that signal distillation and orthogonality regularization are critical and complementary.


\section{Conclusion}
We introduced InstructMoLE to address multi-task interference in parameter-efficient diffusion models by aligning expert routing with the global intent of user instructions. InstructMoLE combines Instruction-Guided Routing (IGR), which selects a globally consistent expert council within each layer, with an output-space orthogonality loss that promotes expert diversity. Experiments on Flux.1 Kontext show that InstructMoLE outperforms LoRA and MoLE variants on challenging multi-conditional benchmarks, establishing global instruction-aware routing as an effective paradigm for faithful compositional control in this setting.

\clearpage

\bibliographystyle{plainnat}
\bibliography{main}

\clearpage

\beginappendix
\raggedbottom

\subsection{The Use of Large Language Models (LLMs)}
We acknowledge the use of a large language model (LLM) to aid in the writing process of this manuscript. Its application was confined to language enhancement tasks, such as proofreading for grammatical and spelling errors, rephrasing sentences to improve clarity and readability, and ensuring a consistent formal tone. The LLM served exclusively as a writing assistant.

\section{Limitations and Societal Impact}
InstructMoLE inherits limitations from the underlying diffusion transformer backbone and from the instruction encoders used to construct the routing signal. Our stress tests indicate failures under severe textual corruption and under high reference load, where several identities or objects must be bound simultaneously. The experiments cover a broad set of multi-conditional generation benchmarks on Flux.1 Kontext, but they do not establish empirical generality across other DiT backbones, domains, languages, or safety-critical deployments. Training also remains computationally expensive despite the parameter-efficient adaptation strategy.

IGR deliberately favors globally coherent skill activation over per-token routing flexibility. This design improves structural consistency, but highly localized instructions with conflicting regional styles may benefit from a constrained hybrid that combines an instance-level base council with limited token-level modulation. Dynamically adapting the number of active experts to instruction complexity is another promising direction that we leave for future work.

InstructMoLE can improve controllable image generation for creative editing, assistive content creation, and data-efficient customization. The same improvements in compositional fidelity and identity preservation could also be misused for deceptive edits, impersonation, or disinformation. Responsible deployment should therefore include consent-aware data practices, provenance or watermarking mechanisms, safety filters for harmful requests, and monitoring for misuse, especially if model weights or generated data are released.

\section{Expert Functional Diversity}
\label{app:diversity}
The orthogonality loss operates on flattened expert outputs, but this does not place the model in a benign high-dimensional random-vector regime. Each expert processes the same input $\mathbf{X}$ through a low-rank LoRA path, so the cosine between two expert outputs depends on the shared input covariance:
\begin{equation}
\cos\left(\mathrm{vec}(\mathbf{X}\mathbf{W}_i), \mathrm{vec}(\mathbf{X}\mathbf{W}_j)\right)
=
\frac{\mathrm{tr}(\mathbf{W}_i^\top \mathbf{X}^\top \mathbf{X}\mathbf{W}_j)}
{\|\mathbf{X}\mathbf{W}_i\|_F \|\mathbf{X}\mathbf{W}_j\|_F}.
\end{equation}
The effective directions are therefore governed by the shared input spectrum and the low-rank bottleneck, not by the full ambient dimensionality. This explains why redundant experts can remain highly aligned unless the model explicitly penalizes functional similarity.

We measure expert diversity on 1,190 real GEdit instructions by comparing pre-gating expert outputs and routing probabilities. Let $\mathbf{U}$ denote the matrix whose rows are the L2-normalized flattened expert output vectors. Adding $\mathcal{L}_{\text{ortho}}$ reduces maximum pairwise redundancy from 0.960 to 0.477 and the Frobenius gap $\|\mathbf{U}^{\top}\mathbf{U}-\mathbf{I}\|_F^2$ from 14.099 to 4.734. Routing balance is preserved: normalized routing entropy increases from 0.938 to 0.960, while the routing coefficient of variation decreases from 0.474 to 0.375. These measurements support the claim that the loss reduces functional redundancy rather than merely adding generic regularization.

\section{Training Data Details}
Our model's versatile capabilities are the direct result of a comprehensive, mixed training dataset, visually summarized in Figure~\ref{fig:train_data}. This dataset is meticulously curated to expose the model to a wide spectrum of conditional inputs, combining large-scale, self-synthesized data for complex compositional tasks with established public datasets for foundational spatial control.

\paragraph{Synthesized Data Corpus.}
To enable sophisticated, instruction-driven editing and composition, we generated a large-scale corpus of training examples. This corpus focuses on tasks that are poorly represented in public datasets but are crucial for real-world applications. It is used only as training data and is not released as a new dataset with this submission; generated samples are filtered automatically for visual quality and instruction alignment. Our synthesis pipeline covers the following key areas:

\begin{itemize}
    \item \textbf{Reference-based Generation:} This category involves tasks that condition the output on multiple input images and a textual instruction.
    \begin{itemize}
        \item \textbf{Face Swapping:} Training triplets consist of a source face image, a target person's image, and an instruction to transfer the facial identity.
        \item \textbf{Multi-Subject Composition:} To teach complex relational reasoning, we generate scenes described by prompts involving multiple, often unrelated, subjects (e.g., ``\textit{An old man, a French Bulldog, a tuba, and a rattlesnake in a sunny park}'').
        \item \textbf{Virtual Try-on:} Samples include a person, multiple clothing items, and potentially a pose skeleton, with instructions to dress the person in the specified apparel.
        \item \textbf{Re-Lighting:} Data consists of a foreground subject and a new background, with the goal of seamlessly integrating the subject into the new lighting environment.
        \item \textbf{Style Transfer:} We provide a content image and a style reference image (e.g., a photograph and a sketch), instructing the model to render the content in the given artistic style.
    \end{itemize}

    \item \textbf{Single Image Editing:} We employ the \textbf{GPT-IMAGE-EDIT-1.5M} dataset~\citep{wang2025gptimageedit15mmillionscalegptgeneratedimage}. This corpus contains over 1.5 million high-quality samples, constructed by leveraging GPT-4o to unify and refine existing public datasets such as OmniEdit, HQ-Edit, and UltraEdit. The refinement process ensures superior visual fidelity and semantic alignment for instruction-based editing tasks.
\end{itemize}

\paragraph{Public Datasets.}
To build a robust foundation in controllable generation, we incorporate several large-scale, publicly available datasets. These datasets provide strong supervision for fundamental spatial and structural alignment tasks.

\begin{itemize}
    \item \textbf{Spatial Alignment Control:} We leverage datasets that provide dense spatial conditioning maps.
    \begin{itemize}
        \item \textbf{Depth Control:} We utilize datasets such as SubjectSpatial-200K, which provide image-depth map pairs, to train the model to generate scenes with accurate spatial layouts.
        \item \textbf{Canny Edge Control:} Similarly, we use datasets with image-Canny edge pairs to enable generation from structural outlines.
        \item \textbf{Pose Control:} We use the widely adopted COCO 2017 dataset~\citep{lin2014microsoft} with its OpenPose keypoint annotations to teach the model to generate human figures conforming to specific pose skeletons.
    \end{itemize}
    
    \item \textbf{Image Inpainting (Fill):} Using datasets like SubjectSpatial-200K, we generate training examples by randomly masking regions of an image. This trains the model to fill in missing parts coherently, a crucial skill for image editing and out-painting tasks.
\end{itemize}

\begin{figure}[t]
  \centering
  \safeincludegraphics[width=0.75\textwidth]{figs/TrainingData.pdf}
  \caption{Training Data Details.}
  \label{fig:train_data}
\end{figure}

\section{Ablation Study}
\begin{table}[!htbp]
\centering
\caption{Ablation study on the MoLE hyper-parameter configuration.}
\label{tab:abla_mole}
\resizebox{\textwidth}{!}{%
\begin{tabular}{ccc|cc|ccc}
\toprule
&\textbf{MoLE} && \textbf{Multi-Subject} ($\uparrow$)  & \textbf{Single-Subject} ($\uparrow$)  & \textbf{Depth RMSE} ($\downarrow$) & \textbf{Canny F1} ($\uparrow$)& \textbf{Pose F1} ($\uparrow$) \\
\midrule
$r$=32  & $N$=8 & $k$=2 & 64.60& \secondbest{6.11}& 35.25  & 35.10\% & 36.47\%\\
$r$=32  & $N$=8 & $k$=4 & \textbf{65.68}   & \textbf{6.15} & \textbf{33.34} & \textbf{41.51\%}& \textbf{40.97\%}   \\
$r$=32  & $N$=8 & $k$=6 & 64.11& 6.10  & \secondbest{33.35}& 37.53\% & \secondbest{37.90\%}  \\
$r$=32  & $N$=8 & $k$=8 & \secondbest{65.22}  & 6.05  & 34.17  & 37.42\% & 34.61\%\\
$r$=64  & $N$=4 & $k$=2 & 65.06& 6.05  & 33.66  & \secondbest{38.27\%}   & 34.77\%\\
$r$=128 & $N$=2 & $k$=1 & 65.10& \textbf{6.15} & 35.50  & 37.52\% & 34.02\%  \\
\bottomrule
\end{tabular}%
}
\end{table}
\paragraph{Analysis of MoLE Hyper-parameters.}
We analyze the MoLE configuration in Table~\ref{tab:abla_mole} to justify our hyper-parameter choices. The results yield two key insights. First, for a fixed expert pool of $N=8$ and rank $r=32$, performance peaks when activating $k=4$ experts. Activating more experts ($k>4$), especially in the non-sparse case ($k=8$), leads to performance degradation, highlighting the importance of sparse activation in mitigating expert interference. Second, for a fixed activation budget ($r \times k \approx 128$), our configuration with a larger pool of diverse, low-rank experts ($N=8, r=32$) is more effective than alternatives with fewer, high-capacity experts (e.g., $N=4, r=64$). This suggests that a greater diversity of specialized experts is more beneficial than individual expert capacity. These findings validate our use of the ($r=32, N=8, k=4$) configuration for all main experiments.

\paragraph{Analysis of Expert Specialization Dynamics.}
\begin{figure}[t]
  \centering
  \safeincludegraphics[width=\textwidth]{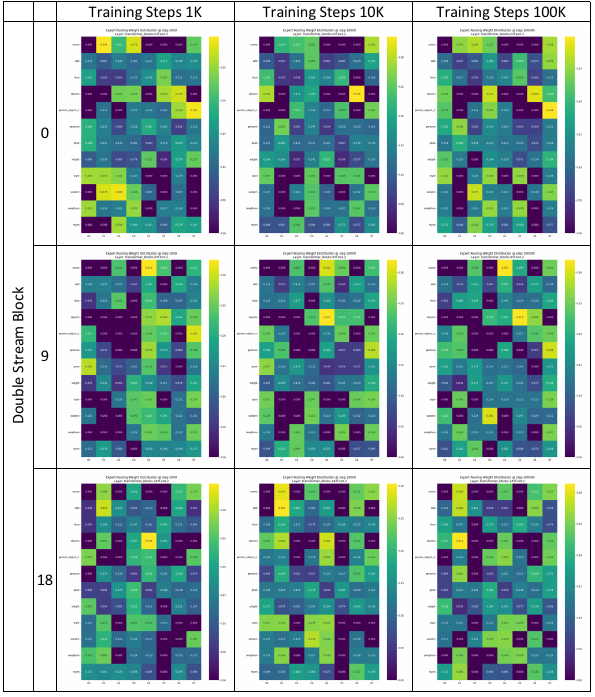}
  \caption{Visualization of the expert routing weight distributions in Double Stream Block. The numbers under the "Double Stream Block" column on the left (0, 9, 18) represent the layer index.}
  \label{fig:double}
\end{figure}

\begin{figure}[t]
  \centering
  \safeincludegraphics[width=\textwidth]{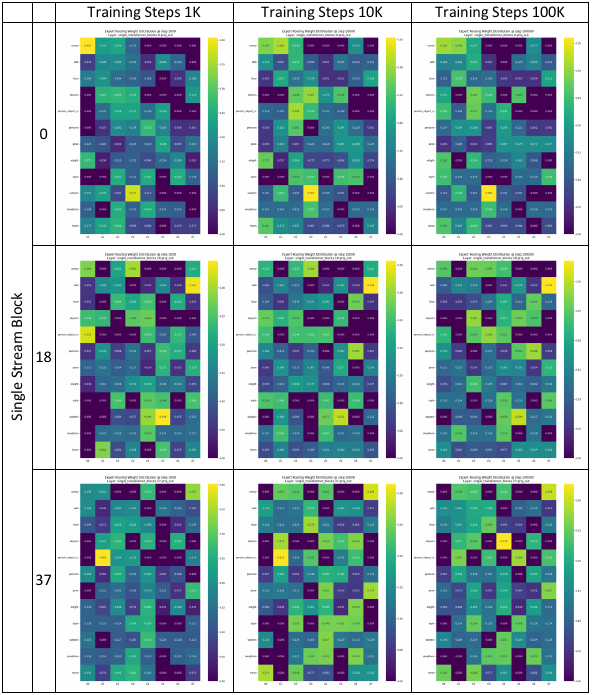}
  \caption{Visualization of the expert routing weight distributions in Single Stream Block. The numbers under the "Single Stream Block" column on the left (0, 18, 37) represent the layer index.}
  \label{fig:single}
\end{figure}
To investigate the learning dynamics of our MoE framework, we visualize the expert routing weight distributions at different training stages (1K, 10K, and 100K steps) and across various model depths, as shown in Figure~\ref{fig:double} and Figure~\ref{fig:single}.
The analysis reveals two key findings.

First, we observe a clear evolutionary trajectory from an initial, diffuse routing policy to a stable and highly specialized one. At 1K steps, the weights are distributed relatively evenly, indicating an early exploratory phase. By 100K steps, the distributions become notably sparse and peaked, with specific experts consistently chosen for particular task categories. This progression from a generalized to a specialized routing policy demonstrates the effective convergence of our training objectives in guiding functional disentanglement.

Second, the nature of expert specialization varies with model depth, suggesting a hierarchical division of labor. In early layers (e.g., Single Stream Block 0), experts tend to specialize in processing fundamental input types, such as distinguishing ``subject" from ``objects". In medial layers (e.g., Single Stream Block 18), we observe the emergence of clear, task-level specialists, with distinct experts strongly favoring tasks like ``canny" or ``swapface". In later layers (e.g., Single Stream Block 37), the routing pattern often becomes more distributed again, suggesting a shift from semantic task execution to a more collaborative final synthesis stage involving multiple experts.


\section{Qualitative Comparison}
\begin{figure}[t]
  \centering
  \safeincludegraphics[width=\textwidth]{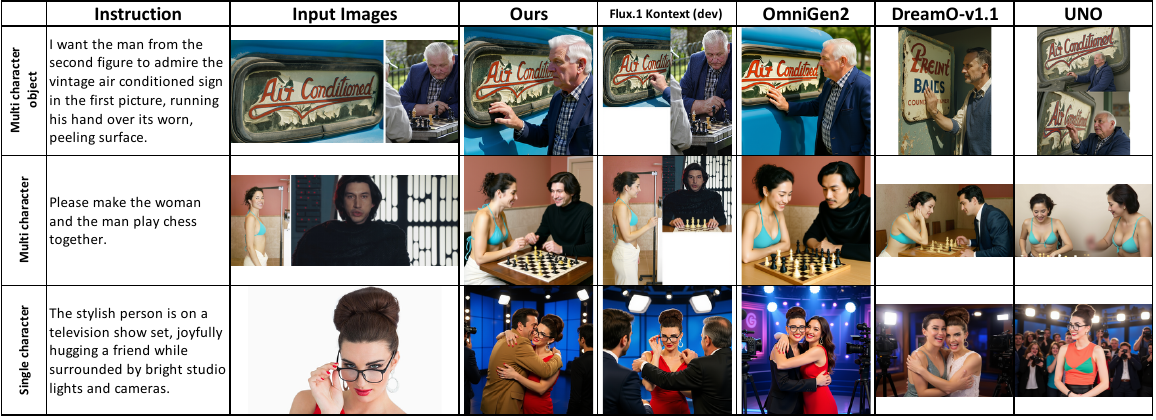}
  \caption{Qualitative comparison of in-context generation on OmniContext benchmark.}
  \label{fig:qual_omni}
\end{figure}

\begin{figure}[t]
  \centering
  \safeincludegraphics[width=\textwidth]{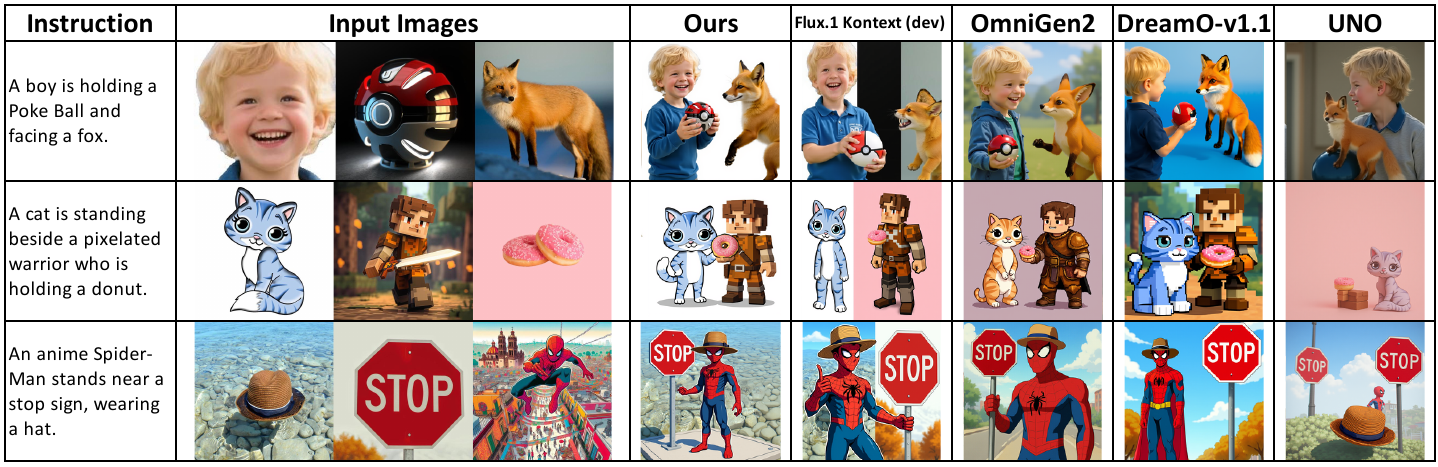}
  \caption{Qualitative comparison of multi-subject driven generation on XVerse benchmark.}
  \label{fig:qual_xverse}
\end{figure}

\begin{figure}[t]
  \centering
  \safeincludegraphics[width=\textwidth]{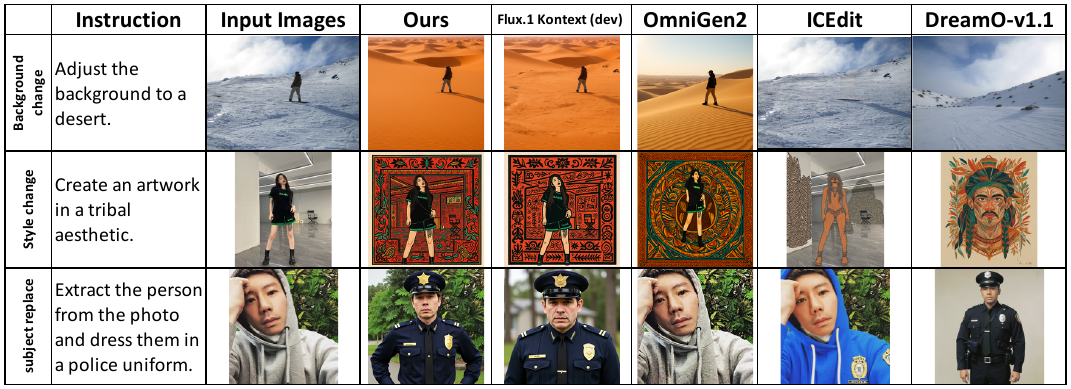}
  \caption{Qualitative comparison of single-image editing on GEdit-EN-full benchmark.}
  \label{fig:qual_gedit}
\end{figure}

\begin{figure}[t]
  \centering
  \safeincludegraphics[width=\textwidth]{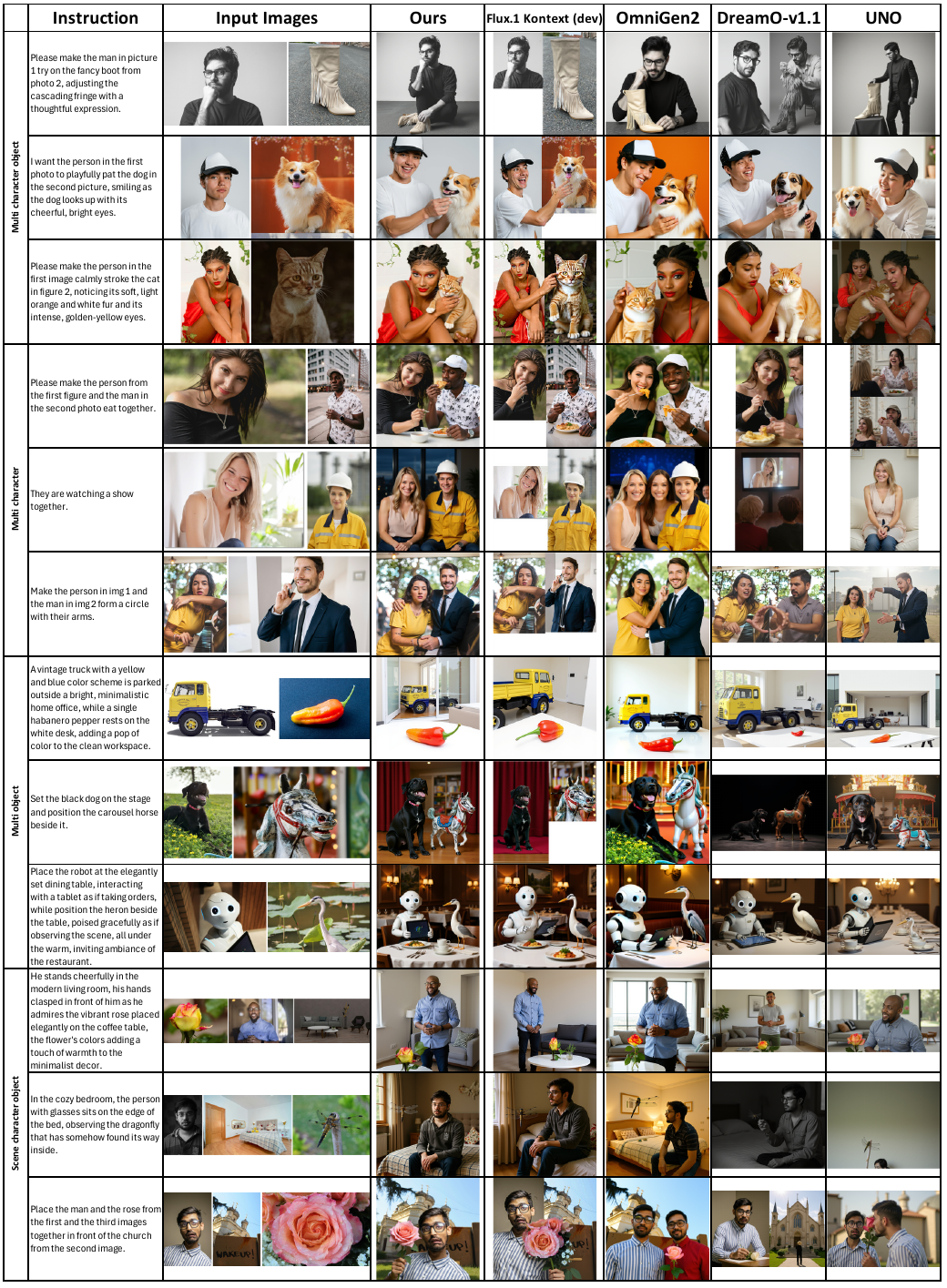}
   \caption{Qualitative comparison on OmniContext benchmark.}
  \label{fig:qual_omni_more}
\end{figure}

\begin{figure}[t]
  \centering
  \safeincludegraphics[width=\textwidth]{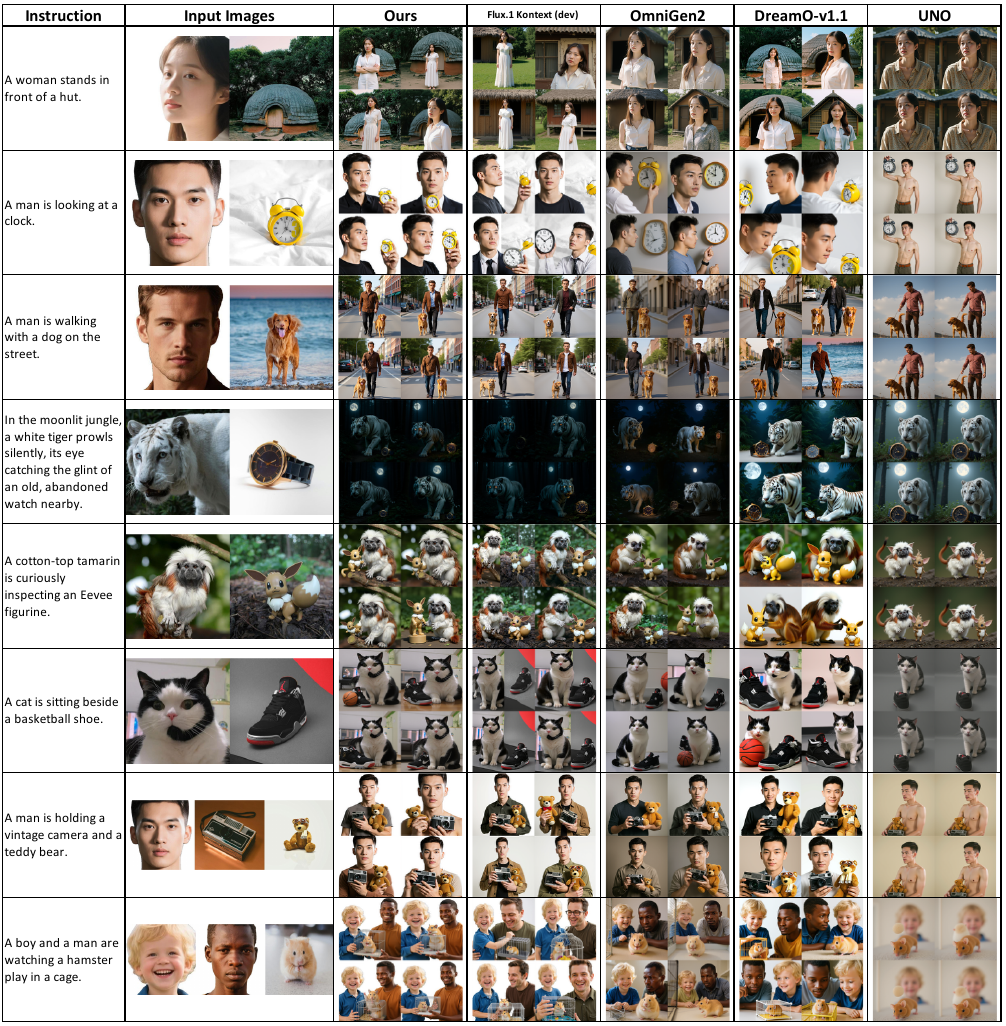}
  \caption{Qualitative comparison on Xverse benchmark.}
  \label{fig:qual_xverse_more}
\end{figure}

\begin{figure}[t]
  \centering
  \safeincludegraphics[width=1\textwidth]{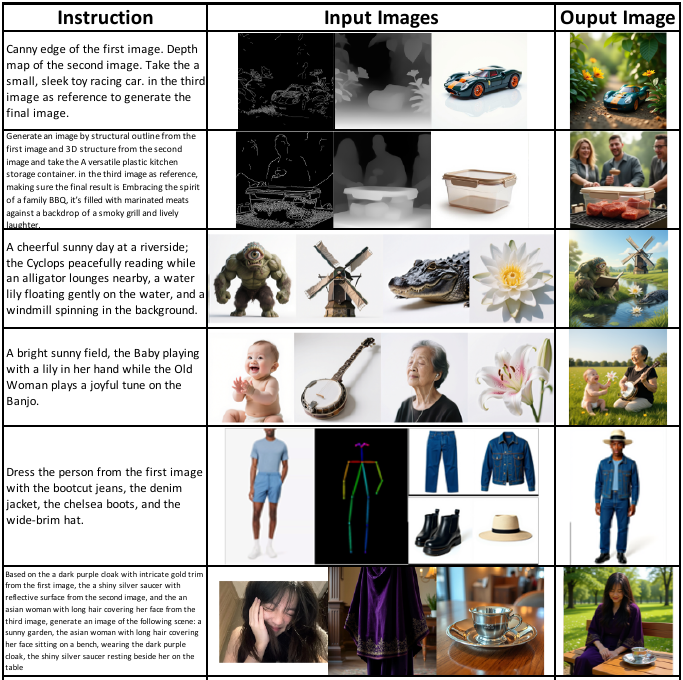}
  \caption{Sample outputs generated by our model (InstructMoLE) from diverse, multi-modal instructions. }
  \label{fig:res-1}
\end{figure}

\section{Qualitative Analysis of Extreme Scenarios}
\label{sec:extreme_analysis}

\begin{figure}[t]
  \centering
  \safeincludegraphics[width=1.0\textwidth]{figs/test_results_comparison.pdf}
  \caption{\textbf{Extreme Scenario Stress Test: Flux.1 Kontext vs. InstructMoLE.} We evaluate robustness across five edge cases spanning input noise, implicit reasoning, and complex composition. While both models share limitations under severe textual corruption (\#1) and high reference load (\#5), InstructMoLE demonstrates superior capabilities in abstract atmospheric inference (\#2) and logical causality (\#4), successfully executing implicit edits where the baseline fails.}
  \label{fig:test_comp}
\end{figure}

To rigorously assess the operational boundaries of our method, we conducted stress tests on extreme scenarios involving textual noise, implicit reasoning, and high-complexity composition, as visualized in Figure~\ref{fig:test_comp}.

\textbf{Limitations and Common Failure Modes.}
Our analysis reveals a discernible threshold for robustness. In Case \#1, characterized by severe typos (e.g., ``Chnage the bckgrnd...''), both InstructMoLE and the baseline fail to resolve the semantic intent. This indicates that while the CLIP semantic anchor provides stability, it cannot fully compensate for T5 token embeddings when textual corruption is excessive. Similarly, in the ultra-complex Case \#5, which requires simultaneous reference binding for three distinct entities (``robot,'' ``boy,'' ``girl''), both models struggle to correctly map identities to the generated subjects. This points to a capacity bottleneck in the underlying cross-attention mechanism of the foundation model when handling high reference loads, a limitation that persists regardless of the routing strategy.

\textbf{Superiority in Reasoning and Coherence.}
Despite these boundary conditions, InstructMoLE exhibits significantly stronger capabilities in \textit{implicit reasoning} and \textit{atmospheric consistency}. In Case \#2 (``Make it look dangerous''), our model successfully modulates the robot's expression and lowers scene lighting to align with the semantic tone, whereas the baseline output remains incongruously bright. Crucially, in Case \#4 (``The person is cold''), InstructMoLE correctly infers the causal implication to generate clothing (a jacket), a logical step the baseline fails to execute. Furthermore, in multi-attribute editing (Case \#3), our method achieves superior spatial and lighting integration of the inserted object.

\end{document}